\documentclass{ecai}
\usepackage{times}
\usepackage{graphicx}
\usepackage{latexsym}

\ecaisubmission   

\usepackage[english]{babel}
\usepackage{amsmath}
\usepackage{amssymb}
\usepackage{amsfonts}
\usepackage[latin1]{inputenc}
\usepackage{graphics}
\usepackage{epic}

\newtheorem{proposition}[theorem]{Proposition}

\newcommand{\bull}{\rule{.85ex}{1ex} \par \bigskip}
\newenvironment{proof}{\noindent {\bf Proof:\ }}{\hfill \bull}
\newcommand{\tuple}[1]{\ensuremath{\langle #1 \rangle}}

\newcommand{\call}[2]{CALL$_{#1,#2}$}

\newcommand{\fmnote}[1]{}

\title{Simple epistemic planning: generalised gossiping}
\author{Martin C. Cooper \ Andreas Herzig \  Faustine Maffre \\
Fr\'{e}d\'{e}ric Maris  \ Pierre R\'{e}gnier\institute{IRIT,
University of Toulouse III, 31062 Toulouse, France, email:
\{cooper$|$herzig$|$maffre$|$maris$|$regnier\}@irit.fr} }

\begin{document}

\maketitle
\bibliographystyle{ecai}

\begin{abstract}
The gossip problem, in which information (known as secrets) must be
shared among a certain number of agents using the minimum number of
calls, is of interest in the conception of communication networks and
protocols. We extend the gossip problem to arbitrary epistemic
depths. For example, we may require not only that all agents know all
secrets but also that all agents know that all agents know all
secrets. We give optimal protocols 
for various versions of the
generalised gossip problem, depending on the graph of communication
links, in the case of two-way communications, one-way communications
and parallel communication. We also study different variants which
allow us to impose negative goals such as that certain agents must
not know certain secrets. We show that in the presence of negative
goals testing the existence of a successful protocol is NP-complete
whereas this is always polynomial-time in the case of purely positive
goals.
\end{abstract}

\section{Introduction}

\fmnote{
Section coupee en deux et remaniee.
}
We consider communication problems concerning $n$ agents. We consider
that initially, for $i=1,\ldots,n$, agent $i$ has some information
$s_i$, also known as this agent's secret since, initially, the other
agents do not know this information. In many applications, this
corresponds to information that agent $i$ wishes to share with all
other agents, such as agent $i$'s signature on a contract or the
dates when agent $i$ is available for a meeting. On the other hand,
it may be confidential information which is only to be shared with a
subset of the other agents, such as agent $i$'s telephone number,
cryptographic key, password or credit card number. More mundanely, it
could simply be some gossip that agent $i$ wants to share. Indeed,
the simplest version of the problem in which all agents want to
communicate their secrets to all other agents (using the minimum
number of communications) is traditionally known as the gossip
problem. Several variants have been studied in the literature, and a
survey of these alternatives and the associated results has been
published \cite{Hedetniemi1988}.

The gossip problem is a particular case of a multiagent epistemic planning problem. 
We view it as an epistemic counterpart of blocksworld problems
where the complexity of epistemic planning problems can be illustrated in a nice way.
\fmnote{
Ajoute.
}
We demonstrate this by studying several variants of the classic problem: 
by supposing that not all pairs of agents can communicate directly, by
allowing parallel or one-way communications, and by introducing the
notion of confidential information which should not be shared with
all other agents. However, our main contribution is to study the
gossip problem at different epistemic depths. In the classic gossip
problem, the goal is for all agents to know all secrets (which
corresponds to epistemic depth 1). The equivalent goal at epistemic
depth 2 is that all agents know that all agents know all the secrets;
at depth 3, all agents must know that all agents know that all agents
know all the secrets. For example, in a commercial setting, if the
secrets are the agents' agreement to the terms of a joint contract,
then an agent may not authorize expenditure on the project before
knowing that all other agents know that all agents agree to the terms
of the contract.
We provide algorithms for these variants and establish their optimality in most of the cases.

The paper is organized as follows. In the next section we formally introduce 
epistemic planning and 
the epistemic version of the classic gossip problem Gossip$_G$($d$).
In Section~\ref{sec:gossip} we study the properties of Gossip$_G$($d$). 
In Section~\ref{sec:one-way} we turn our attention to the version of
this problem in which all communications are one-way (such as e-mails
rather than telephone calls). In Section~\ref{sec:parallel} we study
a parallel version in which calls between different agents can take
place simultaneously. In each of these three cases, we give a
protocol which is optimal (given certain conditions on the graph $G$)
assuming we want to attain all positive epistemic goals up to depth
$d$. We then consider versions of the gossip problem with some
negative goals. This version of the gossip problem has obvious
applications concerning confidential information which is only to be
broadcast to a subset of the other agents. In
Section~\ref{sec:gossip-neg} we show that determining the existence
of a plan which attains a mixture of positive and negative goals is
NP-complete. In Section~\ref{sec:change} we show that allowing agents
to change their secrets (when secrets correspond, for example, to
passwords or telephone numbers) allows more problems to be solved,
but testing the existence of a plan remains NP-complete. We conclude
with a discussion in Section~\ref{sec:discussion}.

\section{Epistemic planning and the gossip problem}
\label{sec:epiPlan}

Dynamic Epistemic Logic \textsf{DEL} \cite{VanDitmarsch2007} provides a formal framework
for the representation of knowledge and update of knowledge, and 
several recent approaches to multi-agent planning are based on it, starting with 
\cite{Bolander2011,Lowe2011}. 
While \textsf{DEL} provides a very expressive framework, 
it was unfortunately proven to be undecidable even for rather simple fragments of the language \cite{AucherB13,Charrier2016a}. 
Some decidable fragments were studied, most of which focused on public events
\cite{Lowe2011,yu2015dynamic}. 
However, the gossip problem requires private communication. 
We here consider a simple fragment of the language of \textsf{DEL}
where the knowledge operator can only be applied to literals.
Similar approaches to epistemic planning can be found in \cite{Kominis2015,Muise2015}.
\fmnote{
Notes sur DEL ajoutees.
}


Here we propose a more direct model.
We use the notation $K_i s_j$ to represent the fact that agent $i$
knows the secret of $j$,
the notation $K_i K_j s_k$ to represent the
fact that agent $i$ knows that agent $j$ knows the secret of $k$, etc.
We use the term \emph{positive fluent} for any epistemic proposition
of the form $K_{i_1} \ldots K_{i_r} s_j$. 
If we consider the secrets
$s_i$ as constants and that agents never forget, then positive
fluents, once true, can never become false. A negative fluent
$\neg(K_{i_1} \ldots K_{i_r} s_j)$ can, of course, become false.
%
Note that these fluents are not modal formulas of epistemic logic;
$K_i$ is not a modal operator. 
They should simply be viewed as independent propositional variables.
\fmnote{
Deux dernieres phrases ajoutees.
}

A \emph{planning problem} consists of an initial state (a set of
fluents $I$), a set of actions and a set of goals (another set of
fluents $Goal$). Each action has a (possibly empty) set of
preconditions (fluents that must be true before the action can be
executed) and a set of effects (positive or negative fluents that
will be true after the execution of the action). A \emph{solution
plan} (or protocol) is a sequence of actions which when applied in this order to
the initial state $I$ produces a state in which all goals in $Goal$
are true. We use the term \emph{epistemic} planning problem when we
need to emphasize that fluents may include the operators $K_i$
($i=1,\ldots,n$). A simple epistemic goal is that all agents know all
the secrets, i.e. $\forall i,j \in \{1,\ldots,n\}$, $K_i s_j$. A
higher-level epistemic goal is $\forall i,j,k \in \{1,\ldots,n\}$,
$K_i K_j s_k$, i.e. that all agents know that all agents know all the
secrets.

The \emph{gossip problem} on $n$ agents and a graph
$G=\tuple{\{1,\ldots,n\},E_G}$  is the epistemic planning problem in
which the actions are \call i j for $\{i,j\} \in E_G$ 
(i.e. there is an edge between $i$ and $j$ in $G$ if and only if they can call each other)
and the
initial state contains $K_i s_i$ for $i=1,\ldots,n$ 
(and implicitly all fluents of the
form $K_{i_1} \ldots K_{i_r} s_j$ with $i_r = j$, together with all fluents of the form
$\neg(K_{i_1} \ldots K_{i_r} s_j)$ with $i_r \neq j$).
The action \call i j has no
preconditions and its effect is that agents $i$ and $j$ share all
their knowledge. We go further and assume that the two agents know
that they have shared all their knowledge, so that, if we had $K_i f$
or $K_j f$ before the execution of \call i j, for any fluent $f$,
then we have $K_{i_1} \ldots K_{i_r} f$ just afterwards, for any $r$
and for any sequence $i_1,\ldots,i_r \in \{i,j\}$.
The assumption
that two agents share all their knowledge when they communicate may
appear unrealistic, but, if the aim is to broadcast information using
the minimum number of calls, this is clearly the best strategy.
Furthermore, in applications in which some secrets must not be
divulged to all agents, it is important to study the worst-case
scenario in which all information is exchanged whenever a
communication occurs.

Let Gossip-pos$_G$($d$) be the gossip problem on a graph $G$ in which
the goal is a conjunction of positive fluents of the form $( K_{i_1} \ldots
K_{i_r} s_j )$ ($1 \leq r \leq d$). Thus, the parameter $d$ specifies
the maximum epistemic depth of goals. We use Gossip$_G$($d$) to
denote the specific problem in which \emph{all} such goals must be
attained. For any fixed $d \geq 1$, Gossip-pos$_G$($d$) can be solved
in polynomial time since it can be coded as a classic STRIPS planning
problem in which actions have no preconditions, and all effects of
actions and all goals are positive \cite{Bylander94}. Indeed, a
necessary and sufficient condition for a solution plan to exist is
that, for all goals $( K_{i_1} \ldots K_{i_r} s_j )$, there is a path
in $G$ from $j$ to $i_1$ passing through $i_r, \ldots, i_2$ (in this
order). Let Gossip-neg$_G$($d$) be the gossip problem in which the
goal is a conjunction of goals of the form $( K_{i_1} \ldots K_{i_r}
s_j )$ ($1 \leq r \leq d$) or $\neg( K_{i_1} \ldots K_{i_r} s_j )$
($1 \leq r \leq d$). We write Gossip-pos($d$), Gossip($d$) and
Gossip-neg($d$) to denote the corresponding problems in which the
graph $G$ is part of the input.
Versions with one-way and parallel communication will be defined in 
sections~\ref{sec:one-way} and~\ref{sec:parallel}. 
In Section~\ref{sec:change} we will define a version where secrets can change truth value.
\fmnote{
Deux dernieres phrases ajoutees.
}

\section{Minimising the number of calls for positive goals}
\label{sec:gossip}

\thicklines \setlength{\unitlength}{1.4pt}
\begin{figure*}[t]
\centering
\begin{picture}(220,60)(0,0)
\put(30,30){\makebox(0,0){$\bullet$}}
\put(10,30){\makebox(0,0){$\bullet$}}
\put(70,30){\makebox(0,0){$\bullet$}}
\put(90,30){\makebox(0,0){$\bullet$}}
\put(130,30){\makebox(0,0){$\bullet$}}
\put(150,30){\makebox(0,0){$\bullet$}}
\put(210,30){\makebox(0,0){$\bullet$}}
\put(190,30){\makebox(0,0){$\bullet$}}
\put(110,10){\makebox(0,0){$\bullet$}}
\put(110,50){\makebox(0,0){$\bullet$}}
\put(54,30){\makebox(0,0){$\cdots$}} \put(166,30){\makebox(0,0){$\cdots$}}
\put(110,10){\line(-1,1){20}}
\put(110,10){\line(1,1){20}}
\put(110,10){\line(-2,1){40}}
\put(110,10){\line(2,1){40}}
\put(110,10){\line(-5,1){100}}
\put(110,10){\line(5,1){100}}
\put(110,10){\line(-4,1){80}}
\put(110,10){\line(4,1){80}}
\put(110,50){\line(-1,-1){20}}
\put(110,50){\line(1,-1){20}}
\put(110,50){\line(-2,-1){40}}
\put(110,50){\line(2,-1){40}}
\put(110,50){\line(-5,-1){100}}
\put(110,50){\line(5,-1){100}}
\put(110,50){\line(-4,-1){80}}
\put(110,50){\line(4,-1){80}}
\put(5,30){\makebox(0,0){$p$}}
\put(215,30){\makebox(0,0){$q$}}
\put(110,56){\makebox(0,0){$1$}}
\put(90,35){\makebox(0,0){$3$}}
\put(130,25){\makebox(0,0){$4$}}
\put(110,4){\makebox(0,0){$2$}}
\end{picture}
\caption{A complete bipartite graph $K_{2,n-2}$.} \label{fig:2-star}
\end{figure*}
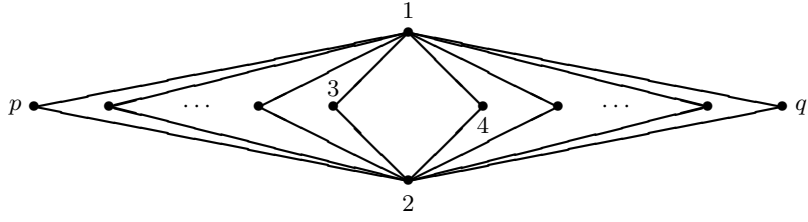

In this section we consider the gossip problem with only positive
goals. The minimal number of calls to obtain the solution of
Gossip$_G$(1) is either $2n - 4$ if the graph $G$ contains a
quadrilateral (a cycle of length 4) as a subgraph, or $2n - 3$ in the
general case \cite{Harary1974}. We first give a simple protocol for
any connected graph before giving protocols requiring many less calls
for special cases of $G$.

\begin{proposition} \label{prop:connected}
If the graph $G$ is connected, then for $n \geq 2$ and $d \geq 1$,
any instance of Gossip-pos$_G$($d$) has a solution of length no
greater than $d(2n-3)$ calls.
\end{proposition}

\begin{proof}
Since $G$ is connected, it has a spanning tree $\mathcal T$. Let the
root of $\mathcal T$ be $1$. Since $n \geq 2$, there is a node $2$
that is connected to $1$. Let $\mathcal T_2$ be the subtree rooted in
$2$ and let $\mathcal T_1$ be the rest of $\mathcal T$, i.e., $1$
together with all its subtrees except $\mathcal T_2$. Let $|\mathcal
T_i| $ be the number of edges in tree $\mathcal T_i$.

Consider the following protocol, composed of a total of $2d$ passes.
Each pass either consists in calls that go upwards in $\mathcal T$
followed by a call between $1$ and $2$, or consists in calls that go
downward.
\begin{tabbing}
\ \ \ \= odd passes: \ \   \= $|\mathcal T_1|$ calls upwards in
$\mathcal T_1$, starting with the leaves; \\ \> \> $|\mathcal T_2|$
calls upwards in $\mathcal T_2$, starting with the leaves; \\ \> \>
\call 1 2
\\
\> even passes:  \>    $|\mathcal T_1|$ calls downwards in $\mathcal
T_1$, starting with $1$; \\ \> \> $|\mathcal T_2|$ calls downwards in
$\mathcal T_2$, starting with $2$
\end{tabbing}

After $k$ passes:
\begin{itemize}
\item
if $k=2m-1$ then $K_1 K_{i_1} \cdots K_{i_{m-1}} s_j$ and $K_2
K_{i_1} \cdots K_{i_{m-1}} s_j $ are true for all $i_1, \ldots, i_{m-1},
j$;
\item
if $k=2m$ then $K_{i_1} \cdots K_{i_m} s_j$ is true for all $i_1, \ldots, i_m,
j$;
\end{itemize}
So the goal is attained after $2d$ passes. Since the odd passes have
$|\mathcal T_1| + |\mathcal T_2| + 1$ $=$ $|\mathcal T|$ $=n-1$ calls and
the even passes have $n-2$ calls, this gives us a total of $d(2n-3)$
calls.
\end{proof}

In fact, for $d \geq 2$, we require considerably less than $d(2n-3)$ calls if $G$ has a Hamiltonian path.

\begin{proposition} \label{prop:Ham}
If the graph $G$ has a Hamiltonian path, then any instance of
Gossip-pos$_G$($d$) has a solution of length no greater than $1 +
(d+1)(n-2)$.
\end{proposition}

\begin{proof}
Let $\pi$ be the Hamiltonian path in $G$. Number the vertices of $G$ from 1 to $n$ in the order they are visited
in $\pi$.    

Consider the following protocol:   
\begin{tabbing}
\ \ \ \= first pass: \ \ \ \ \ \ \ \= \call i {i+1} \= (for $i=n-1,\ldots,1$), \\
               \> \> \ \ \ \ then \ \call i {i+1} (for $i=2,\ldots,n-1$) \\
\> second pass: \>  \call i {i+1}  (for $i=n-2,\ldots,1$) \\   
\> third pass: \> \call i {i+1} (for $i=2,\ldots,n-1$)  \\  
\> \> \> \ \vdots \\
\> even passes: \> \call i {i+1} (for $i=n-2,\ldots,1$) \\  
\> odd passes: \> \call i {i+1} (for $i=2,\ldots,n-1$)  \\   
\> \> \> \ \vdots
\end{tabbing}
It is not difficult to see that the first pass establishes $K_i s_j$
for all $i,j$, and indeed it establishes both $K_{n-1} K_i s_j$ and
$K_n K_i s_j$ for all $i,j$ since \call {n-1} n is the last
communication in this pass. By a straightforward induction argument,
we can show that the $m$th pass, for $m$ even, establishes $K_{i_1}
\ldots K_{i_m} s_j$ for all $i_1, \ldots, i_m, j$, and indeed that
the $m$th pass establishes both $K_1 K_{i_1} \ldots K_{i_m} s_j$ and
$K_2 K_{i_1} \ldots K_{i_m} s_j$ for all $i_1, \ldots, i_m, j$ since
\call 1 2 is the last communication in this pass. Similarly, when
$m$ is odd, the $m$th pass establishes $K_{i_1} \ldots K_{i_m} s_j$
for all $i_1, \ldots, i_m, j$, and indeed both $K_{n-1} K_{i_1}
\ldots K_{i_m} s_j$ and $K_n K_{i_1} \ldots K_{i_m} s_j$ for all
$i_1, \ldots, i_m, j$. The above plan then establishes, after $d$
passes, all possible depth-$d$ epistemic goals. The number of CALL
actions in this plan is $2n-3$ in the first pass and $n-2$ in each
subsequent pass, which makes $2n-3 + (d-1)(n-2) \ = \ 1 + (d+1)(n-2)$
in total after $d$ passes.
\end{proof}

The first pass of the protocol given in the proof of
Proposition~\ref{prop:Ham} scans agents from $n$ to $1$ and then from
$1$ to $n$, whereas each subsequent pass consists of a single scan.
We can explain this by the fact that the purpose of the first scan is
to group secrets together; the purpose of the second scan is then to
broadcast this grouped information to all agents. What is surprising
is that we only require one scan for each subsequent increment in the
epistemic depth $d$.

Determining the existence of a Hamiltonian path is known to be
NP-complete~\cite{GareyJ79}. However, this does not necessarily imply
that finding an optimal solution for Gossip-pos($d$) (the problem in
which the graph $G$ is part of the input) is NP-hard. One reason is
that we do not actually need a Hamiltonian path to obtain a plan of
length $1 + (d+1)(n-2)$. In fact, in the protocol given in the proof
of Proposition~\ref{prop:Ham} we can replace the actions
\call i {i+1} ($i=1,\ldots,n-1$) by any sequence \call {j_i} {i+1} 
($i=1,\ldots,n-1$) such that $j_1=1$ and $\forall i=1,\ldots,n-2$,
$\{j_i,i+1\} \cap \{j_{i+1},i+2\} \neq \emptyset$. Another reason why
the existence of a Hamiltonian path is not necessarily critical is
that we can often actually do better. Indeed, the value  $1 +
(d+1)(n-2)$ is not necessarily optimal, since for certain graphs we
can achieve $(d+1)(n-2)$, i.e. one call less.

The graph shown in Figure~\ref{fig:2-star} is the complete bipartite
graph with parts $\{1,2\}$, $\{3,\ldots,n\}$, and is denoted in graph
theory by $K_{2,n-2}$. We now show that there is a protocol which
achieves $(d+1)(n-2)$ calls provided $G$  contains $K_{2,n-2}$ as a
subgraph. This subsumes a previous result which was given only for
the case of a complete graph $G$~\cite{He2015.6}.

\begin{proposition} \label{prop:split}
For $n \geq 4$, if the $n$-vertex graph $G$ has $K_{2,n-2}$ as a
subgraph, then any instance of Gossip-pos$_G$($d$) has a solution of
length no greater than $(d+1)(n-2)$.
\end{proposition}

\begin{proof}
Suppose that the two parts of $K_{2,n-2}$ are $\{1,2\}$,
$\{3,\ldots,n\}$. We choose an arbitrary partition of the vertices
$3,\ldots,n$ into two non-empty sets $L$, $R$. We can number the
vertices so that $\min(L)=3$ and $\min(R)=4$. Denote $\max(L)$ by $p$
and $\max(R)$ by $q$ (as shown in Figure~\ref{fig:2-star}).

Consider the following protocol:
\begin{tabbing}
\ \ \ \= odd passes: \ \ \ \ \ \ \ \= \call 1 3  \ $\ldots$ \ \call 1 p \ \ \ \= \call 2 4  \ $\ldots$ \ \call 2 q \\
\> even passes: \> \call 1 q \ $\ldots$  \ \call 1 4 \>  \call 2 p \ $\ldots$  \ \call 2 3
\end{tabbing}
In other words: the odd passes are composed of \call 1 x for each
$x \in L$ in increasing order of $x$, followed by \call 2 y for
each $y \in R$ in increasing order of $y$; and the even passes are
composed of \call 1 y for each $y \in R$ in decreasing order of
$y$, followed by \call 2 x for each $x \in L$ in decreasing order
of $x$. The length of this plan after $d+1$ passes is $(d+1)(|L| +
|R|)$ $=$ $(d+1)(n-2)$. It therefore only remains to show that
$(d+1)$ passes are sufficient to establish all possible depth-$d$
epistemic goals. A positive epistemic fluent of the form $K_{i_1}
\ldots K_{i_d} s_j$, for agents $i_1,\ldots,i_d,j$, has depth $d$. In
particular, $s_j$ has depth 0.

For $m \geq 1$, let $H_m$ be the hypothesis that after $m$ passes,
for all depth $m-1$ positive epistemic fluents $f$, we have
\begin{eqnarray*}
(K_1 f \vee K_q f) & \wedge & (K_2 f \vee K_p f) \ \ \ \ {\rm if} \ m \ {\rm is} \ {\rm odd} \\
(K_1 f \vee K_3 f) & \wedge & (K_2 f \vee K_4 f) \ \ \ \ {\rm if} \ m \ {\rm is} \ {\rm even}
\end{eqnarray*}
It is not difficult to see that $H_1$ is true after the first pass.
For $H_m \Rightarrow H_{m+1}$,
suppose $m$ is even.
By $H_m$, after pass $m$, 
we have $K_1 f \vee K_3 f$ for all positive fluents $f$ of depth $m{-}1$.
Thus the first call of pass $m{+}1$,
\call 1 3, makes $1$ and $3$ know all fluents of depth $m{-}1$.
After \call 1 p, $1$ and $p$ know that $1$, $3$, $\ldots$, $p$ know all fluents of depth $m{-}1$.
The same goes for $2$:
since we have $K_2 f \vee K_4 f$ by $H_m$, 
after \call 2 4, $2$ and $4$ know all fluents of depth $m{-}1$.
At the end of pass $m{+}1$
(after \call 2 q), 
$2$ and $q$ know that $2$, $4$, $\ldots$, $q$ know all fluents of depth $m{-}1$.
Thus for any fluent $f$ of depth $m$, 
either $1$ knows $f$ or $q$ knows $f$, and either $2$ knows $f$ or $p$ knows $f$, 
that is, $H_{m+1}$.
The reasoning is similar for $m$ odd.
\fmnote{
Quelques details ajoutes.
}
The above plan therefore establishes, after $d+1$ passes, all possible depth-$d$ epistemic goals.
\end{proof}

Observe that the complete graph on $n \geq 4$ vertices has
$K_{2,n-2}$ as a subgraph. Furthermore, detecting whether an
arbitrary graph $G$ has $K_{2,n-2}$ as a subgraph can clearly be
achieved in polynomial time, since it suffices to test for each pair
of vertices $\{i,j\}$ whether or not $G$ contains all edges of the
form $\{u,v\}$ ($u \in \{i,j\}$, $v \in \{1,\ldots,n\} \setminus
\{i,j\}$).

Recall that Gossip$_G$($d$) denotes the version of
Gossip-pos$_G$($d$) in which the goal consists of all depth-$d$
positive epistemic fluents. We can, in fact, show that the solution
plan given in the proof of Proposition~\ref{prop:split} is optimal
for Gossip$_G$($d$).

\begin{theorem} \label{thm:minpos}
The number of calls required to solve Gossip$_G$($d$) (for any graph
$G$) is at least $(d+1)(n-2)$.
\end{theorem}

\begin{proof}
\fmnote{
Preuve un peu modifiee.
}
Consider any solution plan for Gossip$_G$($d$). The goal of
Gossip$_G$($d$) is to establish $T_{d+1}$ (where $T_r$ is the
conjunction of $K_{i_1} \ldots K_{i_{r-1}} s_{i_r}$ for all
$i_1,\ldots,i_{r} \in \{1,\ldots,n\}$). 

We give a proof by induction.
Suppose that at least $(r+1)(n-2)$ calls are required to establish
$T_{r+1}$. This is true for $r=1$ because it takes at least a
sequence of $2n-4$ calls to establish $T_2$ (each agent knows the
secret of each other agent)~\cite{Baker1972,Hajnal1972,Tijdeman1971}.
For general $r$ and without loss of generality, suppose that before
the last call to establish it, $T_{r+1}$ was false because of lack of
knowledge of agent $j$ (i.e. $K_j T_{r}$ was false). 
By induction hypothesis this is at least the $((r{+}1)(n{-}2){-}1)$-th call.
This call involves $j$ and another agent, say $i$, and establishes
not only $T_{r+1}$, but also $K_j T_{r+1}$ and $K_i T_{r+1}$. 
However, $\neg K_k T_{r+1}$ holds both before and after this call,
for the agents $k$ distinct from $i$ and $j$. 
To establish $T_{r+2}$, it is necessary to distribute $T_{r+1}$
from $i$ and $j$ to other agents and this takes at least $n-2$ calls.
Hence, at least $(r+2)(n-2)$ calls are required in total to establish
$T_{r+2}$. By induction on $r$, it takes at least a sequence of
$(d{+}1)(n{-}2)$ calls to establish $T_{d+1}$.
\end{proof}

\section{One-way communications}  \label{sec:one-way}

We can consider a different version of the gossip problem, which we
denote by Directional-gossip, in which communications are one-way.
Whereas a telephone call is essentially a two-way communication,
e-mails and letters are essentially one-way. We now consider the case
in which the result of \call i j is that agent $i$ shares all his
knowledge with agent $j$ but agent $i$ receives no information from
agent $j$. Indeed, to be consistent with communication by e-mail, in
which the sender cannot be certain when (or even if) an e-mail will
be read by the receiver, we assume that after \call i j, agent $i$
does not even gain the knowledge that agent $j$ knows the information
that agent $i$ has just sent in this call.

Clearly, Directional-gossip-pos$_G$($d$) can be solved in polynomial
time, since any solution  plan for Gossip-pos$_G$($d$) can be
converted into a solution plan for Directional-gossip-pos$_G$($d$) by
replacing each two-way call by two one-way calls. What is surprising
is that the exact minimum number of calls to solve
Directional-gossip-pos$_G$($d$) is often much smaller than this and
indeed often very close to the minimum number of calls required to
solve Gossip-pos$_G$($d$). We consider, in particular, the hardest
version of Directional-gossip-pos$_G$($d$), in which the aim is to
establish all epistemic goals of depth $d$. Let $T_r$ be the
conjunction of $K_{i_1} \ldots K_{i_{r-1}} s_{i_r}$ for all
$i_1,\ldots,i_{r} \in \{1,\ldots,n\}$, and let
Directional-gossip$_G$($d$) denote the directional gossip problem
whose goal is to establish $T_{d+1}$.

In the directional version, the graph of possible communications is
now a directed graph $G$. Let $\overline{G}$ be the graph with the
same $n$ vertices as the directed graph $G$ but with an edge between
$i$ and $j$ if and only if $G$ contains the two directed edges
$(i,j)$ and $(j,i)$. It is known that if the directed graph $G$ is
strongly connected, the minimal number of calls for
Directional-gossip-pos$_G$(1) is $2n - 2$ \cite{Harary1974}. We now
generalise this to arbitrary $d$ under an assumption about the graph
$\overline{G}$.

\begin{proposition}  \label{thm:dirpos}
For all $d \geq 1$, if $\overline{G}$ contains a Hamiltonian path,
then any instance of Directional-gossip-pos$_G$($d$) has a solution
of length no greater than $(d+1)(n-1)$.
\end{proposition}

\begin{proof}
We give a protocol which establishes all positive goals of epistemic
depth up to $d$. Without loss of generality, suppose that the
Hamiltonian path in $\overline{G}$ is $1, 2,
\ldots, n$. 
\fmnote{
D'accord pour le graphe, je n'avais pas bien lu.
J'ai remplac� $1 \rightarrow 2 \rightarrow ... \rightarrow n$ par $1, 2, ..., n$ 
comme l'a sugg�r� Martin car il me semble que c'est plus intuitif.
}
Consider the plan consisting of $d+1$ passes
according to the following protocol:
\begin{tabbing}
odd passes \ \ \ \ \ \ \ \ \= \call i {i+1} (for $i=1,\ldots,n-1$) \\
even passes \> \call {i+1} i (for $i=n-1,\ldots,1$)
\end{tabbing}
We show by a simple inductive proof that this protocol is correct for
any $d \geq 1$. Recall that $T_r$ is the conjunction of $K_{i_1}
\ldots K_{i_{r-1}} s_{i_r}$ for all $i_1,\ldots,i_{r} \in
\{1,\ldots,n\}$. Consider the hypothesis \ H($r$): at the end of pass
$r$, if $r$ is odd we have $K_n T_{r}$ and if $r$ is even we have
$K_1 T_{r}$. Clearly, H(1) is true since at the end of the first pass
agent $n$ knows all the secrets $s_i$ ($i=1,\ldots,n$). If $r$ is odd
and H($r$) holds, then at the end of pass $r+1$, all agents know
$T_r$ and furthermore agent 1 knows this (i.e. $K_1 T_{r+1}$). A
similar argument shows that $H(r) \Rightarrow H(r+1)$ when $r$ is
even. By induction, H($r$) holds for all $r=1,\ldots,d+1$. For $K_n
T_{r}$ or $K_1 T_{r}$ to hold, we must have $T_r$. Thus after $d+1$
passes, and $(d+1)(n-1)$ calls, we have the goal $T_{d+1}$.
\end{proof}

However, as pointed out in Section~\ref{sec:gossip}, determining the
existence of a Hamiltonian path in a graph is
NP-complete~\cite{GareyJ79}.

We now show that the solution plan given in the proof of
Proposition~\ref{thm:dirpos} is optimal even for a complete digraph
$G$.

\begin{theorem} \label{thm:mindirpos}
The number of calls required to solve Directional-gossip$_G$($d$)
(for any digraph $G$) is at least $(d+1)(n-1)$.
\end{theorem}

\begin{proof}
Consider any solution plan for Directional-gossip$_G$($d$). The goal
of Directional-gossip$_G$($d$) is to establish $T_{d+1}$ (the
conjunction of $K_{i_1} \ldots K_{i_{d}} s_{i_{d+1}}$ for all
$i_1,\ldots,i_{d+1} \in \{1,\ldots,n\}$). Consider the following claims
(for $1 \leq r \leq d$):
\begin{description}
\item[C1($r$)] \ after $r(n-1)-1$ calls no agent knows $T_r$
\item[C2($r$)] \  after $r(n-1)$ calls at most one agent knows $T_r$.
\item[C3($r$)] \  at least  $(r+1)(n-1)$ calls are required to establish $T_{r+1}$
\end{description}
C1($1$) is true because $T_1$ is the conjunction of all the secrets $s_j$ and no agent $i$ can know
all the secrets after only $n-2$ calls since after $n-2$ calls, there is necessarily
some agent $i' \neq i$ who has not communicated his secret to anyone.
Let $r \in \{1,\ldots,d\}$.
We will show C1($r$) $\Rightarrow$ C2($r$) $\Rightarrow$ C3($r$) $\Rightarrow$ C1($r+1$).

\paragraph {\bf C1($r$) $\Rightarrow$ C2($r$):} Straightforward, since during one call only one agent gains knowledge.

\paragraph {\bf C2($r$) $\Rightarrow$ C3($r$):} Suppose that C2($r$) holds, i.e.
after $r(n-1)$ calls at most one agent knows $T_r$. This means that the
other $n-1$ agents require some information in order to know $T_r$. Hence we require at least
$n-1$ other calls, i.e. $(r+1)(n-1)$ calls in total, to establish $T_{r+1}$.

\paragraph {\bf C3($r$) $\Rightarrow$ C1($r+1$):}
Suppose C3($r$) is true and C1($r+1$) is false. Then
we require at least  $(r+1)(n-1)$ calls to establish $T_{r+1}$ but after $(r+1)(n-1)-1$ calls some agent $i$
knows $T_{r+1}$. There is clearly a contradiction since agent $i$ cannot know something which is false. \\

\noindent  This completes the proof by induction  that at least $(d+1)(n-1)$
calls are required to establish $T_{r+1}$, since this corresponds exactly to C3($d$).
\end{proof}

It is worth pointing out that, by Theorem~\ref{thm:minpos},
the optimal number of 2-way calls is only $d+1$ less than the optimal
number of one-way calls and is hence independent of $n$, the number
of agents.

\section{Parallel communications}  \label{sec:parallel}

\begin{figure*}[t]
\centering

\begin{picture}(280,105)(0,10)

\put(0,0){
\begin{picture}(70,100)(0,15)
\put(30,100){\dashline{2.5}(0,0)(0,15)}
\put(15,110){\makebox(0,0){\boldmath{$V_1$}}}
\put(45,110){\makebox(0,0){\boldmath{$V_2$}}}
\put(20,30){\makebox(0,0){$\bullet$}}
\put(20,40){\makebox(0,0){$\bullet$}}
\put(20,50){\makebox(0,0){$\bullet$}}
\put(20,60){\makebox(0,0){$\bullet$}}
\put(20,70){\makebox(0,0){$\bullet$}}
\put(20,80){\makebox(0,0){$\bullet$}}
\put(20,90){\makebox(0,0){$\bullet$}}
\put(40,30){\makebox(0,0){$\bullet$}}
\put(40,40){\makebox(0,0){$\bullet$}}
\put(40,50){\makebox(0,0){$\bullet$}}
\put(40,60){\makebox(0,0){$\bullet$}}
\put(40,70){\makebox(0,0){$\bullet$}}
\put(40,80){\makebox(0,0){$\bullet$}}
\put(40,90){\makebox(0,0){$\bullet$}} \put(20,30){\line(1,0){20}}
\put(20,40){\line(1,0){20}} \put(20,50){\line(1,0){20}}
\put(20,60){\line(1,0){20}} \put(20,70){\line(1,0){20}}
\put(20,80){\line(1,0){20}} \put(20,90){\line(1,0){20}}
\put(15,30){\makebox(0,0){$\overline{13}$}}
\put(15,40){\makebox(0,0){$\overline{11}$}}
\put(15,50){\makebox(0,0){$\overline{9}$}}
\put(15,60){\makebox(0,0){$\overline{7}$}}
\put(15,70){\makebox(0,0){$\overline{5}$}}
\put(15,80){\makebox(0,0){$\overline{3}$}}
\put(15,90){\makebox(0,0){$\overline{1}$}}
\put(45,30){\makebox(0,0){$\overline{14}$}}
\put(45,40){\makebox(0,0){$\overline{12}$}}
\put(45,50){\makebox(0,0){$\overline{10}$}}
\put(45,60){\makebox(0,0){$\overline{8}$}}
\put(45,70){\makebox(0,0){$\overline{6}$}}
\put(45,80){\makebox(0,0){$\overline{4}$}}
\put(45,90){\makebox(0,0){$\overline{2}$}}
\end{picture}
}

\put(70,0){
\begin{picture}(70,100)(0,15)
\put(30,100){\dashline{2.5}(0,0)(0,15)}
\put(15,110){\makebox(0,0){\boldmath{$V_1$}}}
\put(45,110){\makebox(0,0){\boldmath{$V_2$}}}
\put(20,30){\makebox(0,0){$\bullet$}}
\put(20,40){\makebox(0,0){$\bullet$}}
\put(20,50){\makebox(0,0){$\bullet$}}
\put(20,60){\makebox(0,0){$\bullet$}}
\put(20,70){\makebox(0,0){$\bullet$}}
\put(20,80){\makebox(0,0){$\bullet$}}
\put(20,90){\makebox(0,0){$\bullet$}}
\put(40,30){\makebox(0,0){$\bullet$}}
\put(40,40){\makebox(0,0){$\bullet$}}
\put(40,50){\makebox(0,0){$\bullet$}}
\put(40,60){\makebox(0,0){$\bullet$}}
\put(40,70){\makebox(0,0){$\bullet$}}
\put(40,80){\makebox(0,0){$\bullet$}}
\put(40,90){\makebox(0,0){$\bullet$}} \put(20,30){\line(1,3){20}}
\put(20,40){\line(2,-1){20}} \put(20,50){\line(2,-1){20}}
\put(20,60){\line(2,-1){20}} \put(20,70){\line(2,-1){20}}
\put(20,80){\line(2,-1){20}} \put(20,90){\line(2,-1){20}}
\put(15,30){\makebox(0,0){$\overline{13}$}}
\put(15,40){\makebox(0,0){$\overline{11}$}}
\put(15,50){\makebox(0,0){$\overline{9}$}}
\put(15,60){\makebox(0,0){$\overline{7}$}}
\put(15,70){\makebox(0,0){$\overline{5}$}}
\put(15,80){\makebox(0,0){$\overline{3}$}}
\put(15,90){\makebox(0,0){$\overline{1}$}}
\put(45,30){\makebox(0,0){$\overline{14}$}}
\put(45,40){\makebox(0,0){$\overline{12}$}}
\put(45,50){\makebox(0,0){$\overline{10}$}}
\put(45,60){\makebox(0,0){$\overline{8}$}}
\put(45,70){\makebox(0,0){$\overline{6}$}}
\put(45,80){\makebox(0,0){$\overline{4}$}}
\put(45,90){\makebox(0,0){$\overline{2}$}}
\end{picture}
}

\put(140,0){
\begin{picture}(70,100)(0,15)
\put(30,100){\dashline{2.5}(0,0)(0,15)}
\put(15,110){\makebox(0,0){\boldmath{$V_1$}}}
\put(45,110){\makebox(0,0){\boldmath{$V_2$}}}
\put(20,30){\makebox(0,0){$\bullet$}}
\put(20,40){\makebox(0,0){$\bullet$}}
\put(20,50){\makebox(0,0){$\bullet$}}
\put(20,60){\makebox(0,0){$\bullet$}}
\put(20,70){\makebox(0,0){$\bullet$}}
\put(20,80){\makebox(0,0){$\bullet$}}
\put(20,90){\makebox(0,0){$\bullet$}}
\put(40,30){\makebox(0,0){$\bullet$}}
\put(40,40){\makebox(0,0){$\bullet$}}
\put(40,50){\makebox(0,0){$\bullet$}}
\put(40,60){\makebox(0,0){$\bullet$}}
\put(40,70){\makebox(0,0){$\bullet$}}
\put(40,80){\makebox(0,0){$\bullet$}}
\put(40,90){\makebox(0,0){$\bullet$}}
\put(20,30){\line(1,2){20}} \put(20,40){\line(1,2){20}}
\put(20,50){\line(1,2){20}} \put(20,60){\line(2,-3){20}}
\put(20,70){\line(2,-3){20}} \put(20,80){\line(2,-3){20}}
\put(20,90){\line(2,-3){20}}
\put(15,30){\makebox(0,0){$\overline{13}$}}
\put(15,40){\makebox(0,0){$\overline{11}$}}
\put(15,50){\makebox(0,0){$\overline{9}$}}
\put(15,60){\makebox(0,0){$\overline{7}$}}
\put(15,70){\makebox(0,0){$\overline{5}$}}
\put(15,80){\makebox(0,0){$\overline{3}$}}
\put(15,90){\makebox(0,0){$\overline{1}$}}
\put(45,30){\makebox(0,0){$\overline{14}$}}
\put(45,40){\makebox(0,0){$\overline{12}$}}
\put(45,50){\makebox(0,0){$\overline{10}$}}
\put(45,60){\makebox(0,0){$\overline{8}$}}
\put(45,70){\makebox(0,0){$\overline{6}$}}
\put(45,80){\makebox(0,0){$\overline{4}$}}
\put(45,90){\makebox(0,0){$\overline{2}$}}
\end{picture}
}

\put(210,0){
\begin{picture}(70,100)(0,15)
\put(30,100){\dashline{2.5}(0,0)(0,15)}
\put(15,110){\makebox(0,0){\boldmath{$V_1$}}}
\put(45,110){\makebox(0,0){\boldmath{$V_2$}}}
\put(20,30){\makebox(0,0){$\bullet$}}
\put(20,40){\makebox(0,0){$\bullet$}}
\put(20,50){\makebox(0,0){$\bullet$}}
\put(20,60){\makebox(0,0){$\bullet$}}
\put(20,70){\makebox(0,0){$\bullet$}}
\put(20,80){\makebox(0,0){$\bullet$}}
\put(20,90){\makebox(0,0){$\bullet$}}
\put(40,30){\makebox(0,0){$\bullet$}}
\put(40,40){\makebox(0,0){$\bullet$}}
\put(40,50){\makebox(0,0){$\bullet$}}
\put(40,60){\makebox(0,0){$\bullet$}}
\put(40,70){\makebox(0,0){$\bullet$}}
\put(40,80){\makebox(0,0){$\bullet$}}
\put(40,90){\makebox(0,0){$\bullet$}} \put(20,30){\line(1,0){20}}
\put(20,40){\line(1,0){20}} \put(20,50){\line(1,0){20}}
\put(20,60){\line(1,0){20}} \put(20,70){\line(1,0){20}}
\put(20,80){\line(1,0){20}} \put(20,90){\line(1,0){20}}
\put(15,30){\makebox(0,0){$\overline{13}$}}
\put(15,40){\makebox(0,0){$\overline{11}$}}
\put(15,50){\makebox(0,0){$\overline{9}$}}
\put(15,60){\makebox(0,0){$\overline{7}$}}
\put(15,70){\makebox(0,0){$\overline{5}$}}
\put(15,80){\makebox(0,0){$\overline{3}$}}
\put(15,90){\makebox(0,0){$\overline{1}$}}
\put(45,30){\makebox(0,0){$\overline{14}$}}
\put(45,40){\makebox(0,0){$\overline{12}$}}
\put(45,50){\makebox(0,0){$\overline{10}$}}
\put(45,60){\makebox(0,0){$\overline{8}$}}
\put(45,70){\makebox(0,0){$\overline{6}$}}
\put(45,80){\makebox(0,0){$\overline{4}$}}
\put(45,90){\makebox(0,0){$\overline{2}$}}
\end{picture}
}
\end{picture}

\caption{The four steps in the first pass of the parallel protocol
for $n=14$.} \label{fig:graphpar1}
\end{figure*}
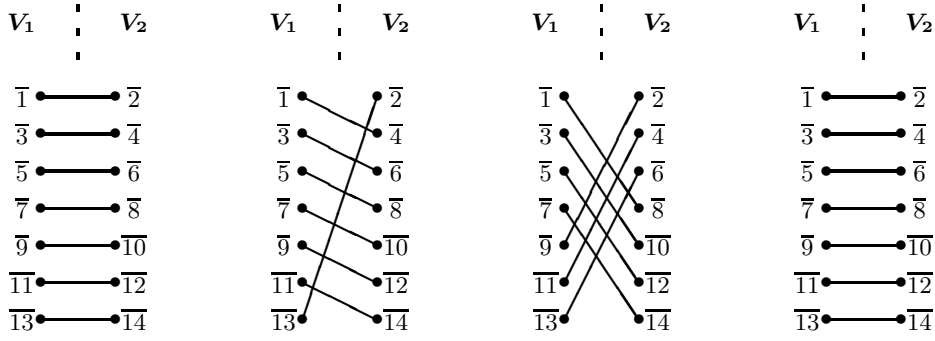

An interesting variant, which we call Parallel-gossip-pos$_G$($d$),
is to consider time steps instead of calls, and thus suppose that in
each time step several calls are executed in parallel. However, each
agent can only make one call in any given time step. We denote by
Parallel-gossip$_G$($d$) the problem of establishing all depth-$d$
positive epistemic fluents. For Parallel-gossip$_G$(1) on a complete
graph $G$, if the number of agents $n$ is even, the time taken (in
number of steps) is $\lceil \log_2 n \rceil$, and if $n$ is odd, it
is $\lceil \log_2 n \rceil + 1$ \cite{Bavelas1950, Landau1954,
Knodel1975}.
We now generalise this to the case of arbitrary epistemic depth $d$.


\begin{proposition}\label{thm:parpos}
For $n \geq 2$, if the $n$-vertex graph $G$ has the complete
bipartite graph $K_{\lceil n/2 \rceil, \lfloor n/2 \rfloor}$ as a
subgraph, then any instance of Parallel-gossip-pos$_G$($d$) has a
solution with  $d(\lceil \log_2 n \rceil-1)+1$ time steps if $n$ is
even, or $d\lceil \log_2 n \rceil+1$ time steps if $n$ is odd.
\end{proposition}

\begin{proof}
Suppose that $G$ has $K_{\lceil n/2 \rceil, \lfloor n/2 \rfloor}$ as
a subgraph. So we can partition the vertex set of $G$ into two
subsets $V_1$ and $V_2$ of size $\lceil n/2 \rceil$ and $\lfloor n/2
\rfloor$, respectively, such that $G$ has an edge $\{i,j\}$ for each
$i \in V_1$ and $j \in V_2$. We can number agents by elements of the
ring $\mathbb{Z}/n\mathbb{Z}=\{\overline{1},\ldots,\overline{n}\}$ so
that for all $i\in\mathbb{Z}$, $\overline{2i+1} \in V_1$  and
$\overline{2i+2} \in V_2$, where $\overline{x}$ denotes the
corresponding element of $\mathbb{Z}/n\mathbb{Z}$ for all
$x\in\mathbb{Z}$. We consider separately the cases $n$ even and $n$
odd.

For even $n$, consider the following protocol:
\begin{tabbing}
\  \= first pass: \\
\>\ \ \  \= For each step $s$ from 1 to $\lceil \log_2 n \rceil$: \\
\> \> \ \ \ \= $\forall i \in \{0, \ldots, (\frac{n}{2}-1)\}$, \call {\overline{2i+1}} {\overline{2i+2^s}} \\ \\
\> subsequent passes: \\
\> \>  Reorder even agents according to the permutation $\pi$ \\
\> \> \> \ \ \ \ given by $\pi(\overline{2i+2^{\lceil \log_2 n \rceil}}) = \overline{2i+2}$; \\
\> \>  Proceed as in the first pass but only for steps $s$ from 2 to $\lceil \log_2 n \rceil$
\end{tabbing}
The first pass of this protocol is illustrated in
Figure~\ref{fig:graphpar1} for $n=14$. Calls are represented by a
line joining two agents.

In the first pass, because of the calls \call {\overline{2i+1}} {\overline{2i+2}}, 
the first step establishes for all
$i\in\mathbb{Z}$, $K_{\overline{2i+1}} s_{\overline{2i+2}}$ and
$K_{\overline{2i+2}} s_{\overline{2i+1}}$. Suppose that after step
$s$, for all $i\in\mathbb{Z}$, we have the conjunction of
$K_{\overline{2i+1}} s_{j}$ and $K_{\overline{2i+2^s}} s_{j}$ for all
$j \in \{\overline{2i+1},\ldots,\overline{2i+2^s}\}$. We have just
seen that this is true for $s=1$ (given that each agent knows his
own secret). In particular, if we replace $i$ by $i+2^{s-1}$ we have
$K_{\overline{2i+2^s+1}} s_{j}$ and $K_{\overline{2i+2^{s+1}}} s_{j}$
for all $j \in \{\overline{2i+2^s+1},\ldots,\overline{2i+2^{s+1}}\}$.
At step $s+1$, we make the calls \call {\overline{2i+1}} {\overline{2i+2^{s+1}}}
for all $i\in\mathbb{Z}$, and this
establishes $K_{\overline{2i+1}} s_{j}$ and
$K_{\overline{2i+2^{s+1}}} s_{j}$ for all $j \in
\{\overline{2i+1},\ldots,\overline{2i+2^{s+1}}\}$. By induction on
$s$, it is easily seen that after $\lceil \log_2 n \rceil$ steps, for
all $i\in\mathbb{Z}$, we have $K_{\overline{2i+1}} s_{j}$ and
$K_{\overline{2i+2}} s_{j}$ for all $j \in \mathbb{Z}/n\mathbb{Z}$.
This means that at the end of the first pass $\forall i,j \in
\{\overline{2i+1},\ldots,\overline{2i+2^{s+1}}\}$, $K_i s_j$.

Let $T_r$ be the conjunction of $K_{j_1} \ldots K_{j_{r-1}} s_{j_r}$
for all $j_1,\ldots,j_{r} \in \mathbb{Z}/n\mathbb{Z}$. We have just
seen that after the first pass $T_2$ is true. Suppose that at the end
of pass $r$, $T_{r+1}$ is true. For the next pass $r+1$,
\call {\overline{2i+1}} {\overline{2i+2^{\lceil \log_2 n \rceil}}} are
the calls in last step of the previous pass $r$. Hence, after
reordering even agents so that $\overline{2i+2^{\lceil \log_2 n
\rceil}}$ replaces $\overline{2i+2}$, we already have for all
$i\in\mathbb{Z}$, $K_{\overline{2i+1}} K_{\overline{2i+2}} T_r$ and
$K_{\overline{2i+2}} K_{\overline{2i+1}} T_r$. We then proceed as for
the first pass replacing $s_j$ by $K_j T_r$ to establish $T_{r+2}$ in
$\lceil \log_2 n \rceil-1$ more steps.

It therefore takes $d$ passes to establish all possible depth-$d$
epistemic goals $T_{d+1}$. The first pass takes $\lceil \log_2 n
\rceil$ steps and the next $d-1$ passes $\lceil \log_2 n \rceil-1$
steps, making a total of $d(\lceil \log_2 n \rceil-1)+1$ steps.

For odd $n$, one can place the first $2^{\lfloor \log_2 n \rfloor}$
agents in a subset $V_{\textit{first}}$, the others being in a subset
$V_{\textit{last}}$ (see the example in Figure~\ref{fig:graphpar} for
$n=13$). Consider the following protocol:
\begin{tabbing}
\ \ \ \= preliminary step: \\
\> \ \ \   \= Each agent in $V_1\cap V_{\textit{last}}$ calls one agent in $V_2\cap V_{\textit{first}}$, \\
\> \> and each agent in $V_2\cap V_{\textit{last}}$ calls one agent in $V_1\cap V_{\textit{first}}$ \\ \\
\> subsequent passes: \\
\>  \> Proceed in $V_{\textit{first}}$ as for the first pass of even case in $\mathbb{Z}/2^{\lfloor \log_2 n \rfloor}\mathbb{Z}$; \\
\> \ \ \   \= Each agent in $V_1\cap V_{\textit{last}}$ calls one agent in $V_2\cap V_{\textit{first}}$, \\
\> \> and each agent in $V_2\cap V_{\textit{last}}$ calls one agent in $V_1\cap V_{\textit{first}}$
\end{tabbing}
A typical pass of this protocol is illustrated in
Figure~\ref{fig:graphpar}. The preliminary step is the step on the
right of this figure.

In the preliminary step, all agents $i_1,\ldots,i_m \in
V_{\textit{last}}$ distribute their knowledge to some agents
$j_1,\ldots,j_m \in V_{\textit{first}}$. Hence, after this step we
have $K_{j_k} s_{i_k}$ for all $k\in\{1,\ldots,m\}$. For each
subsequent pass $r$, it takes $\lfloor \log_2 n \rfloor=\lceil \log_2
n \rceil-1$ steps to distribute knowledge from all agents in
$V_{\textit{first}}$ (hence, in $V_{\textit{last}}$ too because of
the previous step) and establish $K_j T_r$ for all $j \in
V_{\textit{first}}$. Then agents $j_1,\ldots,j_m \in
V_{\textit{first}}$ respectively call the agents $i_1,\ldots,i_m \in
V_{\textit{last}}$ in one more step to establish $T_{r+1}$. These
last calls also establish $K_{j_k} K_{i_k} T_r$ for all
$k\in\{1,\ldots,m\}$ if necessary for the next pass $r+1$.

It takes one preliminary step and $d$ passes of $\lceil \log_2 n
\rceil$ steps to establish all possible depth-$d$ epistemic goals
$T_{d+1}$, which makes a total of $d\lceil \log_2 n \rceil+1$ steps.
\end{proof}

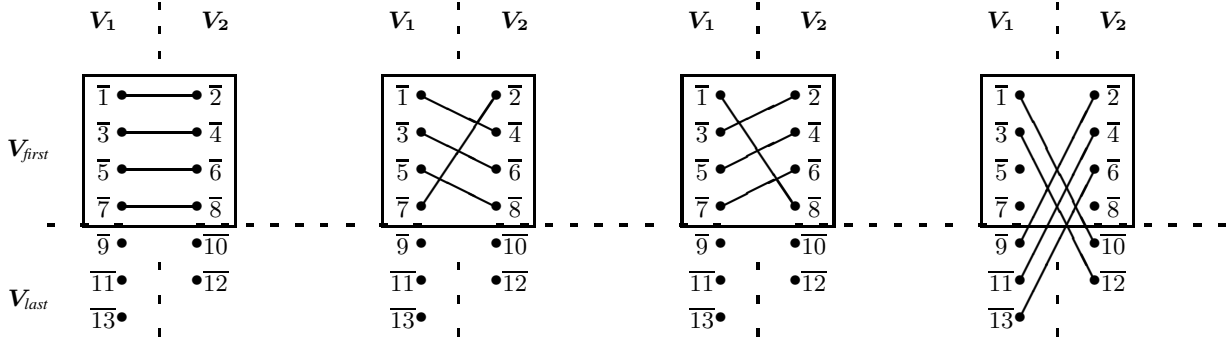
\begin{figure*}[t]
\centering

\begin{picture}(325,95)(-5,20)

\put(-5,0){
\begin{picture}(70,115)(-5,0)
\put(30,100){\dashline{2.5}(0,0)(0,15)}
\put(30,25){\dashline{2.5}(0,0)(0,20)}
\put(15,110){\makebox(0,0){\boldmath{$V_1$}}}
\put(45,110){\makebox(0,0){\boldmath{$V_2$}}}
\put(40,55){\dashline{2.5}(0,0)(35,0)}
\put(0,55){\dashline{2.5}(0,0)(20,0)}
\put(-5,75){\makebox(0,0){\boldmath{$V_{\textit{first}}$}}}
\put(-5,35){\makebox(0,0){\boldmath{$V_{\textit{last}}$}}}
\put(10,55){\framebox(40,40){}} \put(20,30){\makebox(0,0){$\bullet$}}
\put(20,40){\makebox(0,0){$\bullet$}}
\put(20,50){\makebox(0,0){$\bullet$}}
\put(20,60){\makebox(0,0){$\bullet$}}
\put(20,70){\makebox(0,0){$\bullet$}}
\put(20,80){\makebox(0,0){$\bullet$}}
\put(20,90){\makebox(0,0){$\bullet$}}
\put(40,40){\makebox(0,0){$\bullet$}}
\put(40,50){\makebox(0,0){$\bullet$}}
\put(40,60){\makebox(0,0){$\bullet$}}
\put(40,70){\makebox(0,0){$\bullet$}}
\put(40,80){\makebox(0,0){$\bullet$}}
\put(40,90){\makebox(0,0){$\bullet$}}
\put(20,60){\line(1,0){20}} \put(20,70){\line(1,0){20}}
\put(20,80){\line(1,0){20}} \put(20,90){\line(1,0){20}}
\put(15,30){\makebox(0,0){$\overline{13}$}}
\put(15,40){\makebox(0,0){$\overline{11}$}}
\put(15,50){\makebox(0,0){$\overline{9}$}}
\put(15,60){\makebox(0,0){$\overline{7}$}}
\put(15,70){\makebox(0,0){$\overline{5}$}}
\put(15,80){\makebox(0,0){$\overline{3}$}}
\put(15,90){\makebox(0,0){$\overline{1}$}}
\put(45,40){\makebox(0,0){$\overline{12}$}}
\put(45,50){\makebox(0,0){$\overline{10}$}}
\put(45,60){\makebox(0,0){$\overline{8}$}}
\put(45,70){\makebox(0,0){$\overline{6}$}}
\put(45,80){\makebox(0,0){$\overline{4}$}}
\put(45,90){\makebox(0,0){$\overline{2}$}}
\end{picture}}

\put(80,0){
\begin{picture}(70,115)(0,0)
\put(30,100){\dashline{2.5}(0,0)(0,15)}
\put(30,25){\dashline{2.5}(0,0)(0,20)}
\put(15,110){\makebox(0,0){\boldmath{$V_1$}}}
\put(45,110){\makebox(0,0){\boldmath{$V_2$}}}
\put(40,55){\dashline{2.5}(0,0)(35,0)}
\put(0,55){\dashline{2.5}(0,0)(20,0)} \put(10,55){\framebox(40,40){}}
\put(20,30){\makebox(0,0){$\bullet$}}
\put(20,40){\makebox(0,0){$\bullet$}}
\put(20,50){\makebox(0,0){$\bullet$}}
\put(20,60){\makebox(0,0){$\bullet$}}
\put(20,70){\makebox(0,0){$\bullet$}}
\put(20,80){\makebox(0,0){$\bullet$}}
\put(20,90){\makebox(0,0){$\bullet$}}
\put(40,40){\makebox(0,0){$\bullet$}}
\put(40,50){\makebox(0,0){$\bullet$}}
\put(40,60){\makebox(0,0){$\bullet$}}
\put(40,70){\makebox(0,0){$\bullet$}}
\put(40,80){\makebox(0,0){$\bullet$}}
\put(40,90){\makebox(0,0){$\bullet$}}
\put(20,60){\line(2,3){20}} \put(20,70){\line(2,-1){20}}
\put(20,80){\line(2,-1){20}} \put(20,90){\line(2,-1){20}}
\put(15,30){\makebox(0,0){$\overline{13}$}}
\put(15,40){\makebox(0,0){$\overline{11}$}}
\put(15,50){\makebox(0,0){$\overline{9}$}}
\put(15,60){\makebox(0,0){$\overline{7}$}}
\put(15,70){\makebox(0,0){$\overline{5}$}}
\put(15,80){\makebox(0,0){$\overline{3}$}}
\put(15,90){\makebox(0,0){$\overline{1}$}}
\put(45,40){\makebox(0,0){$\overline{12}$}}
\put(45,50){\makebox(0,0){$\overline{10}$}}
\put(45,60){\makebox(0,0){$\overline{8}$}}
\put(45,70){\makebox(0,0){$\overline{6}$}}
\put(45,80){\makebox(0,0){$\overline{4}$}}
\put(45,90){\makebox(0,0){$\overline{2}$}}
\end{picture}
}

\put(160,0){
\begin{picture}(70,115)(0,0)
\put(30,100){\dashline{2.5}(0,0)(0,15)}
\put(30,25){\dashline{2.5}(0,0)(0,20)}
\put(15,110){\makebox(0,0){\boldmath{$V_1$}}}
\put(45,110){\makebox(0,0){\boldmath{$V_2$}}}
\put(40,55){\dashline{2.5}(0,0)(35,0)}
\put(0,55){\dashline{2.5}(0,0)(20,0)} \put(10,55){\framebox(40,40){}}
\put(20,30){\makebox(0,0){$\bullet$}}
\put(20,40){\makebox(0,0){$\bullet$}}
\put(20,50){\makebox(0,0){$\bullet$}}
\put(20,60){\makebox(0,0){$\bullet$}}
\put(20,70){\makebox(0,0){$\bullet$}}
\put(20,80){\makebox(0,0){$\bullet$}}
\put(20,90){\makebox(0,0){$\bullet$}}
\put(40,40){\makebox(0,0){$\bullet$}}
\put(40,50){\makebox(0,0){$\bullet$}}
\put(40,60){\makebox(0,0){$\bullet$}}
\put(40,70){\makebox(0,0){$\bullet$}}
\put(40,80){\makebox(0,0){$\bullet$}}
\put(40,90){\makebox(0,0){$\bullet$}}
\put(20,60){\line(2,1){20}} \put(20,70){\line(2,1){20}}
\put(20,80){\line(2,1){20}} \put(20,90){\line(2,-3){20}}
\put(15,30){\makebox(0,0){$\overline{13}$}}
\put(15,40){\makebox(0,0){$\overline{11}$}}
\put(15,50){\makebox(0,0){$\overline{9}$}}
\put(15,60){\makebox(0,0){$\overline{7}$}}
\put(15,70){\makebox(0,0){$\overline{5}$}}
\put(15,80){\makebox(0,0){$\overline{3}$}}
\put(15,90){\makebox(0,0){$\overline{1}$}}
\put(45,40){\makebox(0,0){$\overline{12}$}}
\put(45,50){\makebox(0,0){$\overline{10}$}}
\put(45,60){\makebox(0,0){$\overline{8}$}}
\put(45,70){\makebox(0,0){$\overline{6}$}}
\put(45,80){\makebox(0,0){$\overline{4}$}}
\put(45,90){\makebox(0,0){$\overline{2}$}}
\end{picture}
}

\put(240,0){
\begin{picture}(70,115)(0,0)
\put(30,100){\dashline{2.5}(0,0)(0,15)}
\put(30,25){\dashline{2.5}(0,0)(0,20)}
\put(15,110){\makebox(0,0){\boldmath{$V_1$}}}
\put(45,110){\makebox(0,0){\boldmath{$V_2$}}}
\put(40,55){\dashline{2.5}(0,0)(35,0)}
\put(0,55){\dashline{2.5}(0,0)(20,0)} \put(10,55){\framebox(40,40){}}
\put(20,30){\makebox(0,0){$\bullet$}}
\put(20,40){\makebox(0,0){$\bullet$}}
\put(20,50){\makebox(0,0){$\bullet$}}
\put(20,60){\makebox(0,0){$\bullet$}}
\put(20,70){\makebox(0,0){$\bullet$}}
\put(20,80){\makebox(0,0){$\bullet$}}
\put(20,90){\makebox(0,0){$\bullet$}}
\put(40,40){\makebox(0,0){$\bullet$}}
\put(40,50){\makebox(0,0){$\bullet$}}
\put(40,60){\makebox(0,0){$\bullet$}}
\put(40,70){\makebox(0,0){$\bullet$}}
\put(40,80){\makebox(0,0){$\bullet$}}
\put(40,90){\makebox(0,0){$\bullet$}}
\put(20,30){\line(1,2){20}} \put(20,40){\line(1,2){20}}
\put(20,50){\line(1,2){20}} \put(20,80){\line(1,-2){20}}
\put(20,90){\line(1,-2){20}}
\put(15,30){\makebox(0,0){$\overline{13}$}}
\put(15,40){\makebox(0,0){$\overline{11}$}}
\put(15,50){\makebox(0,0){$\overline{9}$}}
\put(15,60){\makebox(0,0){$\overline{7}$}}
\put(15,70){\makebox(0,0){$\overline{5}$}}
\put(15,80){\makebox(0,0){$\overline{3}$}}
\put(15,90){\makebox(0,0){$\overline{1}$}}
\put(45,40){\makebox(0,0){$\overline{12}$}}
\put(45,50){\makebox(0,0){$\overline{10}$}}
\put(45,60){\makebox(0,0){$\overline{8}$}}
\put(45,70){\makebox(0,0){$\overline{6}$}}
\put(45,80){\makebox(0,0){$\overline{4}$}}
\put(45,90){\makebox(0,0){$\overline{2}$}}
\end{picture}
}

\end{picture}

\caption{The four steps in each pass of the parallel protocol for
$n=13$. The step on the right also occurs on its own as a preliminary
step.} \label{fig:graphpar}
\end{figure*}

It is worth pointing out that determining whether a $n$-vertex graph
$G$ has the complete bipartite graph $K_{\lceil n/2 \rceil, \lfloor
n/2 \rfloor}$ as a subgraph can be achieved in polynomial time. To
see this, firstly observe that any pair of vertices $i,j$ of $G$
which are not joined by an edge must be in the same part in the
complete bipartite graph. In linear time, we can partition the
vertices of $G$ into subsets $S_1,\ldots,S_r$ such that vertices
$i,j$ not joined by an edge in $G$ belong to the same set $S_t$ (for
some $1 \leq t \leq r$). It only remains to test whether it is
possible to partition the numbers $|S_1|,\ldots,|S_r|$ into two sets
whose sums are $\lceil n/2 \rceil$ and $\lfloor n/2 \rfloor$. This
partition problem can be solved by dynamic programming in $O(r(|S_1|
+ \cdots + |S_r|))$ time and space, which is at worst quadratic since
$r \leq n$ and $|S_1| + \cdots + |S_r| =
n$~\cite{DBLP:conf/ijcai/Korf09}. On the other hand, it is known that
deciding whether Directional-gossip($1$) (the problem in which the
digraph $G$ is part of the input) can be solved in a given number of
steps is NP-complete~\cite{KrummeCV92}.

We now show that the solution plans given in the proof of
Proposition~\ref{thm:parpos} are optimal in the number of steps.

\begin{theorem} \label{thm:par}
The number of steps required to solve Parallel-gossip$_G$($d$) (for
any graph $G$) is at least $d(\lceil \log_2 n \rceil-1)+1$ if $n$ is
even, or $d\lceil \log_2 n \rceil+1$ if $n$ is odd.
\end{theorem}

\begin{proof}
\fmnote{
Preuve un peu modifiee (sur le modele de la preuve du theoreme 4).
(Et le $T_{r+1}$ que j'avais oublie corrige !)
}
Consider any solution plan for Parallel-gossip$_G$($d$). Recall that
$T_r$ is the conjunction of $K_{i_1} \ldots K_{i_{r-1}} s_{i_r}$ for
all $i_1,\ldots,i_{r} \in \{1,\ldots,n\}$.

We give a proof by induction. For even $n$, suppose that at least
$r(\lceil \log_2 n \rceil-1)+1$ steps are required to establish
$T_{r+1}$. This is true for $r=1$ because it takes at least a sequence
of $\lceil \log_2 n \rceil$ steps of calls for knowledge from
any agent to reach $n$ agents (thus establishing $T_2$)
\cite{Bavelas1950, Landau1954, Knodel1975}.
For general $r$ and without loss of
generality, suppose that before the last step to establish it,
$T_{r+1}$ was false because of lack of knowledge of agent $j$ (i.e.
$K_j T_{r}$ was false). 
By induction hypothesis this is at least the $(r(\lceil \log_2 n \rceil{-}1))$-th step.
A call in this step involves $j$ and another agent, say $i$, and establishes
not only $T_{r+1}$, but also $K_j T_{r+1}$ and $K_i T_{r+1}$. 
However, $\neg K_k T_{r+1}$ holds both before and after this step,
for the agents $k$ distinct from $i$ and $j$. 
%
To establish $T_{r+2}$, it is necessary to
distribute $T_{r+1}$ from $i$ and $j$ to all other agents and this
takes at least $\lceil \log_2 n \rceil-1$ steps (since each step can at most double the
number $m$ of agents having this knowledge and thus
$\lceil \log_2 (n/2) \rceil$ steps are required to go from $m=2$ to $m=n$). Hence, at least
$(r+1)(\lceil \log_2 n \rceil-1)+1$ steps are required to establish
$T_{r+2}$.  By induction on $r$, we obtain the lower bound $d(\lceil
\log_2 n \rceil-1)+1$.

For odd $n$, the proof is similar but at least one more step is
required for each epistemic level $r$ because at least one agent
doesn't communicate his knowledge on the first step to establish
$T_{r+1}$. Hence, it takes at least a sequence of $\lceil \log_2 n
\rceil+1$ steps for knowledge from all $n$ agents to reach each
others, and the lower bound is $d\lceil \log_2 n \rceil+1$.
\end{proof}

It is interesting to note that it can happen that increasing the
number of secrets (and hence the number of agents) leads to less
steps. Consider the concrete example of 7 or 8 agents. By
Proposition~\ref{thm:parpos} and Theorem~\ref{thm:par}, the number of
steps decreases from $3d+1$ to $2d+1$ when the number of agents
increases from 7 to 8. We can explain this by the fact that in the
case of an odd number of agents, during each step there is
necessarily one agent who is not communicating. By adding an extra
agent, we can actually achieve a larger number of calls in a fewer
number of steps.

\section{Complexity of gossiping with negative goals}  \label{sec:gossip-neg}

Not surprisingly, when we allow negative goals, the gossip problem
becomes harder to solve. However, we will show that for several
different versions of this problem, we avoid the PSPACE complexity of
classical planning~\cite{Bylander94}.

We also consider a slightly more general version of the gossip
problem in which the maximum epistemic depth $d$ is no longer a
constant, but is part of the input. Let Gossip-pos and Gossip-neg be,
respectively, the same as Gossip-pos($d$) and Gossip-neg($d$) in
which there is no fixed bound $d$ on the maximum epistemic depth of
goal fluents. Although we do not specify the exact format in which
the goals are given, we make the assumption that this requires at
least $n+d+m$ space, where $m$ is the number of goal fluents. Recall
that in these versions of the gossip problem, the graph $G$ is also
part of the input.

\begin{theorem}
Gossip-pos $\in$ P. Indeed, if a solution plan exists, it can be found in polynomial time.
\end{theorem}

\begin{proof}
The connected components of the graph $G$ can be determined in
polynomial time as can a spanning tree of each connected component.
If there is a fluent $K_{i_1} \ldots K_{i_r} s_j$ in $Goal$, where
the agents $i_1,\ldots,i_r,j$ do not all belong to the same connected
component of $G$, then the planning problem has no solution.
Otherwise, there is a solution obtained by applying the protocol
given in the proof of Proposition~\ref{prop:connected} to each
connected component and for a value of $d$ equal to the maximum
epistemic depth of goals. To construct this solution we only require
knowledge of the spanning tree of each connected component.
\end{proof}

When we allow negative goals the problem of deciding the existence of a solution plan
becomes NP-complete.

\begin{theorem} \label{thm:unbounded}  \label{thm:NPC}
Gossip-neg and Gossip-neg(1) are both NP-complete.
\end{theorem}

\begin{proof}
We first show that Gossip-neg $\in$ NP.
We will show that if a solution plan $P$ exists then there is a solution plan $P'$
of length no greater than $md(n-1)$, where $m$ is the number of goal fluents
and $d$ the maximum epistemic depth of goal fluents. The validity of a plan
of this length can clearly be verified in polynomial time.

Consider a goal $g = K_{i_1} \ldots K_{i_r} s_j$ (where $r \leq d$).
In $P$ there must be a sequence of \call p q actions where the
edges $\{p,q\}$ in the graph $G$ form a path from $j$ to $i_1$
passing through $i_{r}, \ldots, i_{2}$ in this order. There may be
many such paths: for each goal $g$ let $path(g)$ be one such path.
Divide $path(g)$ into subpaths $j \rightarrow i_r$, $i_r \rightarrow
i_{r-1}$, $\ldots$, $i_2 \rightarrow i_1$. If any of these subpaths
contains a cycle, this cycle can be eliminated from $path(g)$. Call
the resulting reduced path $path'(g)$. We can see that each subpath
in $path'(g)$ is of length no greater than $n-1$ (otherwise it would
contain a cycle). Thus, $|path'(g)| \leq r(n-1) \leq d(n-1)$. Each
goal $g$ can therefore be achieved by a subset of the actions of $P$
(corresponding to $path'(g)$). Let $P'$ be identical to $P$ except
that we only keep the actions \call p q such that the corresponding
edge $\{p,q\}$ belongs to some $path'(g)$. $P'$ then constitutes a
valid plan and is of length at most $md(n-1)$. It follows that
Gossip-neg $\in$ NP since the validity of a plan of this length can
be verified in polynomial time. Trivially, we also have Gossip-neg(1)
$\in$ NP since Gossip-neg(1) is a subproblem of Gossip-neg.

To complete the proof, it suffices to give a polynomial reduction from
the well-known NP-complete problem SAT to Gossip-neg(1).
Let $I_{\rm SAT}$ be an instance of SAT.
We will construct a graph $G$ and a list of goals such that the corresponding instance $I_{\rm Gossip}$
of Gossip-neg(1) is equivalent to $I_{\rm SAT}$.
Recall that the nodes of $G$ are the agents and the edges of $G$ the communication links
between agents.

For each propositional variable $x$ in $I_{\rm SAT}$, we add four
nodes $x$, $\overline{x}$, $b_x$, $d_x$ to $G$ joined by the edges
shown in Figure~\ref{fig:SATred}(b). There is a source node $a$ in
$G$ and edges $(a,x)$, $(a,\overline{x})$ for each variable $x$ in
$I_{\rm SAT}$. For each clause $c_j$ in $I_{\rm SAT}$, we add a node
$c_j$ joined to the nodes corresponding to the literals of $c_j$.
This is illustrated in Figure~\ref{fig:SATred}(a) for the clause $c_j
= \overline{x} \vee y \vee z$. The solution plan to $I_{\rm Gossip}$
will make the secret $s_a$ transit through $x$ (on its way from $a$
to some clause node $c_j$) if and only if $x=true$ in the
corresponding solution to $I_{\rm SAT}$.

For each clause $c_j$ in $I_{\rm SAT}$, $G$ contains a clause gadget
as illustrated in Figure~\ref{fig:SATred}(a) for the clause
$\overline{x} \vee y \vee z$. We also add $K_{c_j} s_a$ to the set of
goals. Clearly, the secret $s_a$ must transit through one of the
nodes corresponding to the literals of $c_j$ ($\overline{x}$, $y$ or
$z$ in the example of Figure~\ref{fig:SATred}) to achieve the goal
$K_{c_j} s_a$.

To complete the reduction, it only remains to impose the constraint that $s_a$ transits through
at most one of the nodes $x$, $\overline{x}$, for each variable $x$ of $I_{\rm SAT}$.
This is achieved by the negation gadget shown in Figure~\ref{fig:SATred}(b) for each variable $x$.
We add the goals $K_{d_x} s_{b_x}$, $\neg(K_{d_x} s_a)$ for each variable $x$,
and the goal $\neg(K_{c_j} s_{b_x})$ for each variable $x$ and each clause $c_j$
(containing the literal $x$ or $\overline{x}$).
The goal  $K_{d_x} s_{b_x}$ ensures that the secret $s_{b_x}$ transits through $x$
or $\overline{x}$.
Suppose that $s_{b_x}$ transits through $x$:
then $s_a$ cannot transit through $x$ before $s_{b_x}$ (because of the goal $\neg(K_{d_x} s_a)$)
and cannot transit through $x$ after $s_{b_x}$ (because of the goal $\neg(K_{c_j} s_{b_x})$).
By a similar argument, if  $s_{b_x}$ transits through $\overline{x}$,
then $s_a$ cannot transit through $\overline{x}$. Thus, this gadget imposes
that $s_a$ transits through exactly one of the the nodes $x$, $\overline{x}$.

We have shown that $I_{\rm SAT}$ has a solution if and only if $I_{\rm Gossip}$ has a solution.
Since the reduction is clearly polynomial, this completes the proof.
\end{proof}

\thicklines \setlength{\unitlength}{1.8pt}

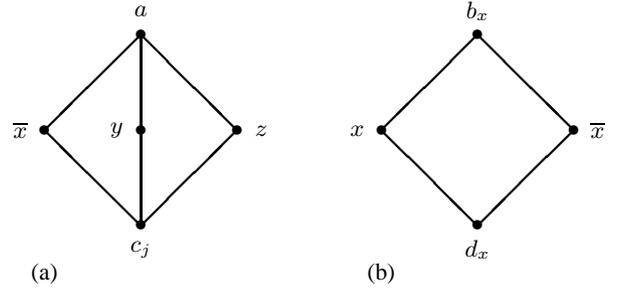
\begin{figure}[t]
\centering

\begin{picture}(130,75)(0,0)

\put(0,0){
\begin{picture}(60,65)(10,15)
\put(20,50){\makebox(0,0){$\bullet$}}
\put(40,30){\makebox(0,0){$\bullet$}}
\put(40,50){\makebox(0,0){$\bullet$}}
\put(60,50){\makebox(0,0){$\bullet$}}
\put(40,70){\makebox(0,0){$\bullet$}}
\put(40,30){\line(-1,1){20}}
\put(40,30){\line(0,1){20}}
\put(40,30){\line(1,1){20}}
\put(40,70){\line(-1,-1){20}}
\put(40,70){\line(0,-1){20}}
\put(40,70){\line(1,-1){20}}
\put(15,50){\makebox(0,0){$\overline{x}$}}
\put(35,50){\makebox(0,0){$y$}}
\put(65,50){\makebox(0,0){$z$}}
\put(40,75){\makebox(0,0){$a$}}
\put(40,25){\makebox(0,0){$c_j$}}

\put(20,20){\makebox(0,0){(a)}}
\end{picture}
}

\put(70,0){
\begin{picture}(60,65)(10,15)
\put(20,50){\makebox(0,0){$\bullet$}}
\put(40,30){\makebox(0,0){$\bullet$}}
\put(60,50){\makebox(0,0){$\bullet$}}
\put(40,70){\makebox(0,0){$\bullet$}}
\put(40,70){\line(-1,-1){20}}
\put(40,70){\line(1,-1){20}}
\put(20,50){\line(1,-1){20}}
\put(60,50){\line(-1,-1){20}}
\put(15,50){\makebox(0,0){$x$}}
\put(40,25){\makebox(0,0){$d_x$}}
\put(65,50){\makebox(0,0){$\overline{x}$}}
\put(40,75){\makebox(0,0){$b_x$}}

\put(20,20){\makebox(0,0){(b)}}
\end{picture}
}

\end{picture}

\caption{(a) gadget imposing the clause $c_j \ = \ \overline{x} \vee y \vee z$;
(b) gadget imposing the negation $\overline{x} = \neg x$.}
\label{fig:SATred}
\end{figure}

Our NP-completeness results are not affected by a restriction to
one-way communication, i.e. Directional-gossip-neg and
Directional-gossip-neg($1$) are both NP-complete, by exactly the same
proof as for Theorem~\ref{thm:unbounded}. A similar remark holds for
Parallel-gossip-neg and Parallel-gossip-neg($1$).

\section{Complexity of gossiping with variable secrets}  \label{sec:change}

\thicklines \setlength{\unitlength}{1.4pt}
\begin{figure*}[t]
\centering
\begin{picture}(220,105)(0,0)
\put(0,100){\line(1,0){220}}
\put(187,85){\makebox(0,0){\textbf{PSPACE-complete}}}
\put(48,85){\fbox{STRIPS planning}} \put(0,71){\line(1,0){220}}
\put(195,50){\makebox(0,0){\textbf{NP-complete}}}
\put(45,58){\fbox{Gossip-neg-change}} \put(55,40){\fbox{Gossip-neg}}
\put(0,30){\line(1,0){220}} \put(0,29.5){\line(1,0){220}}
\put(55,13){\fbox{Gossip-pos}}
\put(197,14){\makebox(0,0){\textbf{polynomial}}}
\put(0,0){\line(1,0){220}}
\end{picture}
\caption{Complexity results for different decision versions of the gossip problem.}
\label{fig:gossip-sat}
\end{figure*}

Up to now we have assumed that the secrets $s_i$ are constants.
We now introduce a new kind of action CHANGE$_i$ which simulates what happens 
when agent $i$ changes his secret (which we imagine corresponds, for example, to his
password). The effect of action CHANGE$_i$ is to render all fluents
of the form $K_{i_1},\ldots,K_{i_r} s_i$ false,
for $i_r \neq i$, since agent $i_r$
does not know the new value of $s_i$.
\fmnote{
Phrase un peu modifiee. 
}
%
These new actions allow us to solve certain gossip problems which cannot be solved without them. For example,
consider two agents and the set of goals $\{K_1 s_2, \neg K_2 s_1\}$. In Gossip-neg
there is no solution to this planning problem, since the goal $K_1 s_2$ requires the action \call 1 2
which also establishes $K_2 s_1$. However, the plan
(\call 1 2, CHANGE$_1$) achieves the goals $K_1 s_2$ and $\neg K_2 s_1$.
An example of this plan is exchanging telephone numbers with someone
and then promptly changing one's own number.
Denote by Gossip-neg-change the version of Gossip-neg with the new CHANGE$_i$ actions.
Although the CHANGE$_i$ actions can help to solve more problems, it turns out that
Gossip-neg-change is in the same complexity class as Gossip-neg, as we now prove.

\begin{theorem}
Gossip-neg-change is NP-complete.
\end{theorem}

\begin{proof}
It is simple to verify that the reduction from SAT given in the proof of
Theorem~\ref{thm:NPC} remains valid:  in the instances corresponding to instances of SAT,
the actions CHANGE$_a$ and CHANGE$_{b_x}$ cannot be used without
destroying goals which must be attained.

Thus, to complete the proof, it suffices to show that Gossip-neg-change $\in$ NP.
As in the proof of Theorem~\ref{thm:unbounded}, it suffices to show that if a solution plan $P$
exists, then there is a solution plan $P'$ of length no greater than a polynomial function
of $m$, $d$ and $n$.
To transform $P$ into an equivalent plan $P'$, we can eliminate all useless actions.
We consider an action $a$ to be useless in $P$ if all fluents $K_{i_1} \ldots K_{i_r} s_j$
it achieves were already true or CHANGE$_j$ occurs after $a$ in $P$.
Since fluents  $K_{i_1} \ldots K_{i_r} s_j$ can only become true at most once after the last
occurrence of CHANGE$_j$ in $P$, we can deduce that the number of actions in $P'$
is bounded above by $md(n-1)$ (as in the proof of Theorem~\ref{thm:unbounded}).
If CHANGE$_i$ occurs in $P$, then all  its occurrences except the last can be deleted
without affecting the validity of the plan. Thus the total number of actions in $P'$
is bounded above by $n+md(n-1)$, which completes the proof.
\end{proof}

In the problem Gossip-neg-change, the CHANGE$_i$ actions have no
preconditions. If there are different actions CHANGE$_i$ depending on
the values of some subset of the secrets, then it is not difficult to
see that we can simulate the version of classical STRIPS planning in
which all actions have a single effect, which is known to be
PSPACE-complete~\cite{Bylander94}. A more interesting avenue of
future research is perhaps to investigate restricted versions of
Gossip-neg or Gossip-neg-change which can be solved in polynomial
time. As a simple example, suppose that the agents can be arranged in
a hierarchy so that each agent $i$ belongs to a level $L_i$ and the
goal is to communicate all secrets upwards in the hierarchy but not
downwards. A solution consists in, for each level $L$ in turn
starting with the lowest level, all agents at this level communicate
their secrets to all agents at level $L+1$ in the hierarchy, then all
agents at level $L+1$ change their secrets so that the agents at
level $L$ no longer know these secrets. In this way all secrets
percolate up the hierarchy but not down.

\section{Discussion and conclusion}  \label{sec:discussion}

We summarize our complexity results in Figure~\ref{fig:gossip-sat}.
In each case the problem is the decision problem, i.e. testing the
existence of a solution plan. The general conclusion that can be
drawn from this figure is that many interesting epistemic planning
problems are either solvable in polynomial time or are NP-complete,
thus avoiding the PSPACE-complete complexity of planning. We consider
the gossip problem to be a foundation on which to base the study of
richer epistemic planning problems involving, for example,
communication actions with preconditions involving the contents of
the messages received by the agent. Previous work on temporal
planning may help to provide a more realistic model of communication
actions in which, for example, the length of a call is a function of
the quantity of information exchanged, and correct communication
during a telephone call requires concurrency of the speaking and
listening actions of the two agents~\cite{CushingKMW07,CooperMR13}.

Restricting our attention to the epistemic version of the classical
gossip problem in which all positive epistemic goals of depth $d$
must be attained, we have generalised many results from the classical
gossip problem to the epistemic version. We have shown that for a
complete graph $G$, no protocol exists which solves Gossip$_G$($d$)
in less than $(d+1)(n-2)$ calls. This was known to be true for
$d=1$~\cite{Baker1972,Hajnal1972}. We have given a protocol which
uses only this number of calls (for any graph $G$ containing
$K_{2,n-2}$ as a subgraph). In the case of one-way communications, we
have again generalised the optimal protocol from the classical gossip
problem to the epistemic version. This protocol requires only
$(d+1)(n-1)$ calls. When calls can be performed in parallel, and the
aim is to minimise the number of steps rather than the number of
calls, we have again generalised the optimal protocol from the
classical gossip problem to the epistemic version. In this case, only
$O(d \log n)$ steps are required.

There remain many interesting open problems concerning the
optimisation version of the gossip problem: given any graph $G$,
determine the minimum number of calls required to attain a set of
goals. For example, in the case of one-way communications, our
optimal protocol requires a Hamiltonian path in a graph and detecting
a Hamiltonian path is NP-complete~\cite{GareyJ79}. However, it may be
that another optimal protocol exists which does not require the
existence of a Hamiltonian path. A similar situation occurred in the
case of two-way communications, in which we gave a protocol which
depends on the existence of  $K_{2,n-2}$ as a subgraph, and this
graph can be detected in polynomial time. The complexity of the
problem of minimising the number of calls (whether two-way or
one-way) in an arbitrary graph $G$ is still open.

In this paper we have assumed a centralised approach in which a
centralised planner decides the actions of all agents. Other workers
have studied the classical gossip problem from a completely different
perspective, assuming that all agents are
autonomous~\cite{AttamahDGH14,DitmarschEPRS15,GrossiEtal16}. 
An interesting avenue of future
research would be to consider the generalised gossip problem in this
framework.

\bibliography{biblio}

\end{document}